\documentclass[letterpaper]{article} %

\usepackage{biblatex} 
\usepackage[utf8]{inputenc}
\usepackage{amsfonts}
\usepackage{amsmath}
\usepackage{pgf}
\usepackage{import}
\usepackage{hyperref}
\usepackage{amsthm}

\theoremstyle{definition}
\newtheorem{definition}{Definition}[section]

\theoremstyle{remark}
\newtheorem*{remark}{Remark}

\newtheorem{theorem}{Theorem}[section]

\newtheorem{lemma}[theorem]{Lemma}

\usepackage{listings}
\usepackage{xcolor}

\definecolor{codegreen}{rgb}{0,0.6,0}
\definecolor{codegray}{rgb}{0.5,0.5,0.5}
\definecolor{codepurple}{rgb}{0.58,0,0.82}
\definecolor{backcolour}{rgb}{0.98,0.98,0.99}

\lstdefinestyle{mystyle}{
    backgroundcolor=\color{backcolour},   
    commentstyle=\color{codegreen},
    keywordstyle=\color{red},
    numberstyle=\tiny\color{codegray},
    stringstyle=\color{gray},
    basicstyle=\ttfamily\footnotesize,
    breakatwhitespace=false,         
    breaklines=true,                 
    captionpos=b,                    
    keepspaces=true,                 
    numbers=left,                    
    numbersep=5pt,                  
    showspaces=false,                
    showstringspaces=false,
    showtabs=false,                  
    tabsize=2
}

\usepackage[ruled,vlined]{algorithm2e}

\begin{document}

\title{e3nn: Euclidean Neural Networks}
\author{Mario Geiger and Tess Smidt}
\date{January 2021}

\begin{abstract}
    We present \texttt{e3nn}, a generalized framework for creating E(3) equivariant trainable functions, also known as Euclidean neural networks. \texttt{e3nn} naturally operates on geometry and geometric tensors that describe systems in 3D and transform predictably under a change of coordinate system. The core of \texttt{e3nn} are equivariant operations such as the \texttt{TensorProduct} class or the spherical harmonics functions that can be composed to create more complex modules such as convolutions and attention mechanisms. These core operations of \texttt{e3nn} can be used to efficiently articulate Tensor Field Networks, 3D Steerable CNNs, Clebsch-Gordan Networks, SE(3) Transformers and other E(3) equivariant networks. 
\end{abstract}

\maketitle

\section{Introduction}

{\color{red} This document is a draft.}


3D Geometry (e.g. points in 3D space, volumes, meshes or surfaces) and geometric tensors (e.g. vector fields or strain fields) are traditionally challenging data types to use for machine learning because coordinates and coordinate systems are sensitive to the symmetries of 3D space: 3D rotations, translations, and inversion. One of the motivations for incorporating symmetry into machine learning models on 3D data is to eliminate the need for data augmentation -- the 500-fold increase in brute-force training necessary for a model to learn 3D patterns in arbitrary orientations. 

Invariant models are more straightforward to implement; they only operate on scalars which interact thought simple scalar multiplication. Invariant models are restricted in what they can express. The way physical quantities transform under rotations is classified into the so-called \textit{representations} of rotations, which are indexed by the integer $l$. Scalars corresponds to $l=0$ (they don't change under rotation), vectors to $l=1$, and higher order quantities like for instance the d, f, g, ... orbitals transform with higher $l=2, 3, 4, \dots$. Only scalars ($l=0$) are invariant under rotation.
Invariant models can't predict a $l>0$ quantity that would change with the rotation of the input data.

Sometimes in order to articulate a task in an equivariant manner requires reframing the question -- but this effort is often rewarded because reframing can highlight hidden symmetry-breaking biases that often occur when equivariance is not imposed (e.g. the fact that linear algebra solvers report specific eigenvectors for degenerate subspaces when there is truly a degeneracy of bases within that subspace).

We think equivariant models are not just useful to predict $l>0$ quantities or the beauty of articulating the task in an equivariant way. It is true that models restricted to $l=0$ and $l=1$ are expressive enough to be universal\cite{universality_tfn_dymhaggai2020,scalar} which means the they are able to fit any equivariant function of output $l\leq 1$.
However machine learning in addition to expressivity also requires generalization and data efficiency.
Generalization and data efficiency can be measured by the so called learning curve -- the generalization error as a function of the size of the training set.
It has been observed that models including internal representations of order $l=2$ display a better learning curve than models only containing $l\leq 1$ internal representations \cite{md_batzner2021,NeuralScaling2022}.

Equivariance is both elegant and simple, however, implementing equivariant neural networks can be technical and error prone. This is why we developed a python library that allows researchers to conveniently explore the realm of equivariant neural networks for 3D data.
In this paper, we introduce \texttt{e3nn} a software framework for articulating E(3) equivariant models. 
We have built \texttt{e3nn} to take care of the intricacies of geometric tensor algebra so the user can focus on building trainable modules for diverse scientific and engineering domains. We describe in detail the provably general mathematical primitives that we have developed and how they can be easily composed to form the modules of a wide variety of $E(3)$-equivariant architectures. This article serves as a resource for the relevant mathematics for understanding and developing these types of networks.

\texttt{e3nn} is born from the two papers \cite{3dsteerable,tfn}. In its first versions it had many different functions used to build different kind of equivariant layers. In January 2021, all these operations has been fused into an unified and powerful parameterizable operation called \texttt{TensorProduct}.




\section{Related work}

Euclidean neural networks encompass the class of Euclidean symmetry equivariant neural networks introduced in Refs. \cite{tfn,3dsteerable,kondor2018clebsch}.

There is a long history of using symmetry considerations to construct machine learning representations and models in the molecular and materials sciences. It is common to use Behler-Parinello symmetry functions \cite{behler-2007-gener-neural} ACE \cite{ace} or SOAP kernels \cite{soap} to craft invariant representations of local atomic environments that are used to make per atom predictions of energy.

Currently, most symmetry-aware neural networks used in the molecules and materials community are invariant models, where all operations of the trainable model act on scalar quantities \cite{schnet, schnet2, CGCNN}. DimeNet uses angular information between triplets of atoms by pre-computing these angles and using the resulting scalars as part of neural network operations \cite{dimenet}.

There also exist restricted equivariant methods, that use only scalar or vector operations on scalar and vector representations; these are a strict subset of the operations allowed with an equivariant network \cite{En_gnn, equiv_message_passing}.
In some cases, these restricted operations are adequate for good performance on learning tasks, but these implementations do not span the full space of possible equivariant operations and representations.

There are also equivariant kernel methods where all equivariant operations are contained in the construction of kernel \cite{PhysRevLett.120.036002}.

Recently, Ref. \cite{finzi21_2104.09459} have developed methods for deriving equivariant linear layers for general groups when provided the group generators.

Many papers have documented improved accuracy on training tasks when going from invariant to equivariant models, even within the same framework \cite{miller2020,md_batzner2021}. 


An important goal for \texttt{e3nn} is to provide a flexible framework to express the most general equivariant operations for all forms of 3D data. All methods discussed above can be succinctly articulated with the operations of \texttt{e3nn}.

So far, \texttt{e3nn} and Euclidean neural networks have been applied to diverse applications:
finding symmetry breaking order parameters \cite{SmidtGeiger2021},
predicting molecular properties \cite{miller2020},
prediction of phonon density of states \cite{phonons_Chen2021},
prediction of the interatomic potentials \cite{md_batzner2021},
developing deep learning for diffusion MRI \cite{mri_muller2021},
predicting the formation factor and effective permeability from micro-CT images \cite{https://doi.org/10.48550/arxiv.2104.05608},
approximating steady-state fluid flow around a fixed particle \cite{Siddani_2021},
defining a graph neural network for which node and edge attributes are not restricted to invariant
scalars \cite{https://doi.org/10.48550/arxiv.2110.02905},
generate coarse-graining of molecular conformations \cite{https://doi.org/10.48550/arxiv.2201.12176} and
simulate 100 million atoms \cite{https://doi.org/10.48550/arxiv.2204.05249}.

\section{Our Contribution}

In \texttt{e3nn} we have developed or implemented the following: (i) parity-equivariance such that our models are equivariant to $O(3)$ rather than only $SO(3)$ (Section~\ref{sec:parity}), (ii) a general \texttt{TensorProduct} class that allows complete flexibility creating trainable equivariant bi-linear operations between irreps (Section~\ref{sec:tp}), (iii) the \texttt{ReduceTensorProduct} class which can decompose tensors of arbitrary rank (with extra index symmetry constraints like for instance $x_{ij} = -x_{ji}$) onto a direct sum of irreps of $O(3)$ (Section~\ref{sec:rtp}), Fast conversion between irreps and signal on the sphere (Section~\ref{sec:s2}) and (iv) we have coded all the base operations needed to build equivariant modules of arbitrary complexity (Section~\ref{sec:equiv_modules}).

We emphasize that these methods are useful for any tensor computations, not only in a deep learning context.

\section{Group Representations} \label{sec:math}

\subsection{Group}
\begin{definition}
A group is a mathematical structure made of a set $G$ and a composition law $G\times G\to G$ that obey a list of rules. (i) It contains an identity element $e\in G$ such that $ae=ea=a$ for all $a\in G$, (ii) It obeys associativity $(ab)c=a(bc)$ and (iii) each element has an inverse $a a^{-1}=a^{-1} a=e$. (See \cite{zee} for a more rigorous definition and discussion about left and right inverses).
\end{definition}
The set of all 3D rotations form a group.
The set of all 3D translations form a group.
The set containing the identity and the inversion form a small group too.
The set made of all the composition of those transformations (rotations, translations and inversion) form the Euclidean Group.

\subsection{Representations}
An important perspective shift to make when using Euclidean neural networks is that all data in the networks are typed by how the data transforms under the Euclidean Group.
The mathematical structure that describe such a type are the \textit{group representations}.
We will restrict ourselves to finite dimensional representations.

\begin{definition}
    For a given group $G$.
    A function $D : G \longrightarrow \mathbb{R}^{d\times d}$ that maps each element of the group to a $d\times d$ matrix is a representation iff it follows the structure of the group: $D(e)=1$, $D(ab)=D(a)D(b)$ (See again \cites{zee} for a more rigorous definition.)
\end{definition}

\begin{definition}
An irreducible representation is a representation that does not contain a smaller representation in it. i.e. there is no nontrivial projector $P\in\mathbb{R}^{q\times d}$ such that the function $g\mapsto P D(g) P^t$ is a representation.
\end{definition}

The data types of \texttt{e3nn} are the irreducible representations (irreps) of $O(3)$.
Another way to think of this is that in order for our network to preserve transformation properties of our data, we must tell our network how the data transforms in the first place.
We use irreducible representations because they are the smallest representation and complete \cite{zee}, i.e any finite dimensional representation of $O(3)$ can be decomposed on to the irreducible representations of $O(3)$.
The irreps of $O(3)$ factorize into the irreps of $SO(3)$ and the irreps of inversion (also called parity).

\subsection{Irreducible representations of rotations}

The irreps of $SO(3)$ are indexed by the integers $0, 1, 2, \dots$, we call this index $l$. 
The $l$-irrep is of dimension $2l+1$.
This is shown for instance in \cite{rao}.\footnote{The demonstration is done starting from the commutation relations of the generators of rotation and introducing raising and lowering operators.}
$l=0$ (dimension 1) corresponds to \textit{scalars} and $l=1$ (dimension 3) corresponds to \textit{vectors}.
Higher $l$-irrep are less common, the best example for $l=2$ is the decomposition of a symmetric rank 2 tensor.



\subsection{Example of the symmetric matrix}

A symmetric $3\times3$ Cartesian tensor comprises of 9 numbers but only 6 degrees of freedom (the 3 numbers on the diagonals and 3 numbers that fill the off diagonal entries). It transforms under rotation as 
\begin{equation}
    x \longrightarrow R x R^t
\end{equation}
where $R$ here is some rotation matrix.
Keep in mind that although $x$ and $R$ are both $3\times3$ matrices, they are very different objects. The former represent a physical quantity measured in a given coordinate frame and the later represent a rotation between two coordinate frames.

We can express the matrix $x$ in terms of the irreducible representations of $SO(3)$. A symmetric $3\times3$ Cartesian matrix has a trace which is a single scalar $l=0$-irrep and 5 numbers comprising the symmetric traceless components which define the traceless part of the diagonal and off-diagonal elements $l=2$-irrep. With the irreducible representations, it is more clear that there are only 6 degrees of freedom. In the \texttt{e3nn} framework we denote the irreps of $x$ as \texttt{1x0e + 1x2e}. 
\texttt{1x} means there is one time the irrep, \texttt{+} denotes a direct sum of irreps and \texttt{e} denotes the parity of the irreps, that we introduce now.

\subsection{Parity} \label{sec:parity}
Parity acts on 3D geometry by inversion: $\vec x \longrightarrow -\vec x$.
Mirror and glide can be decomposed into a special rotation followed by an inversion. Or inversion followed by a special rotation since inversion and rotations commute.
There is two irreducible representations of parity: the even representation ($p=1$) and the odd representation ($p=-1$).
For instance, the velocity vector is odd, when changing the coordinates with  $\vec x \longrightarrow -\vec x$, the velocity vector will change its sign.
However, the cross product of two vectors give rise to an even quantity, $p=1$ i.e. a pseudovector.
A pseudovector will not change under inversion of the coordinate frame.

\subsection{Equivariance}

For a group $G$, a function $f: X \rightarrow Y$ is equivariant under $G$ if $f(D_X(g) x) = D_Y(g) f(x)$, $\forall g \in G, x 
\in X$ where $D_X(g)$ (resp. $D_Y(g)$) is the representation of the group element $g$ acting on the vector space $X$ (resp. $Y$). Another way to say this is that we can act with the group on the inputs or the outputs of the function and are guaranteed to get the same answer.

For Euclidean neural networks, all individual network operations are required to obey this equivariance where $G = E(3)$, the Euclidean group in 3 dimensions. By requiring individual components to be equivariant, we guarantee the equivariance of the entire neural network, for $f: X_1 \rightarrow X_2$ and $h: X_2 \rightarrow X_3$, if both functions are equivariant, their composition is as well equivariant: $h(f(D(g)x)) = h(D(g)f(x)) = D(g) h(f(x))$.

A function with learned parameters can be abstracted as $f: W\times X \rightarrow Y$, where $w \in W$ is a choice of learned parameters (or weights). We require the parameters to be scalars, meaning that they don't transform under a transformation of $E(3)$: $f(w, D_X(g)x) = D_Y(g) f(w, x)$. Note, this implies that weights are scalars and are invariant under any choice of coordinate system.

\section{Irreducible representations used in e3nn} \label{sec:repr}
Irreducible representations are unique up to a change of basis. If $D^l(R)$ is an irreducible representation of $SO(3)$ of order $l$ then for any invertible matrix $A$,
\begin{equation}
    \tilde D^l(R) = A D^l(R) A^{-1}
\end{equation}
is also a valid irreducible representation of $SO(3)$.
To remove this degeneracy people usually look for $D^l(R)$ such that the generator of rotation around an arbitrary chosen axis is diagonal. The generator $J_y^l$ is defined as the representation of the infinitesimal rotation around axis $y$,
\begin{equation}
    J_y^l = \lim_{\alpha \to 0} \frac{D^l(\text{$\alpha$ around $y$}) - D^l(e)}{\alpha}
\end{equation}
Diagonalizing $J_y^l$ for each $l$ gives a nice set of basis for the irreps $D^l$. In this basis the entries of the matrices $D^l$ are complex numbers. This is done for instance in \cite{rao}.

The representations of $SO(3)$ are real, in the sense that they are equivalent to their conjugate, there exists a change of basis $A$ such that,
\begin{equation}
    D^l(R)^* = A D^l(R) A^{-1}
\end{equation}
with $A$ symmetric, $A=A^T$ therefore there exists a change of basis $S$ in which $S D^l(R)^* S^{-1} = S D^l(R) S^{-1}$, see Chapter II.4 of \cite{zee}.

Therefore in \texttt{e3nn}, for memory and computation reasons, we opted to use real numbers.
We get the representations in the usual basis using a code from QuTip\cite{qutip} then we change basis using
\begin{equation} \label{eq:real_complex}
    z^l_m = \left\{ \begin{array}{ll}
        (-i)^l \frac1{\sqrt2} (x^l_{|m|} - i x^l_{-|m|}) & m < 0 \\
        (-i)^l & m=0 \\
        (-1)^m (-i)^l \frac1{\sqrt2} (x^l_{|m|} + i x^l_{-|m|}) & m > 0
    \end{array} \right.
\end{equation}

\section{Spherical harmonics} \label{sec:sh}

\subsection{Definition}

\begin{definition}
The spherical harmonics are a family of functions $Y^l$ from the unit sphere to the irrep $D^l$.
(We later relax the definition and defined them on $\mathbb{R}^3$.)
For each $l=0,1,2,\dots$ the spherical harmonics can be seen as a vector of $2l+1$ functions $Y^l(\vec x) = (Y^l_{-l}(\vec x), Y^l_{-l+1}(\vec x), \dots, Y^l_{l}(\vec x))$.
Each $Y^l$ is equivariant to SO(3) with respect to the irrep of the same order, i.e.
\begin{equation}
    Y^l_m(R \vec x) = \sum_{n=-l}^l D^l(R)_{mn} Y^l_n(\vec x)
\end{equation}
where $R$ is any rotation matrix and $D^l$ are the irreducible representation of SO(3).
They are normalized $\|Y^l(\vec x)\|=1$ when evaluated on the sphere ($\|\vec x\|=1$). Different normalization exists, this one is named \texttt{norm} in \texttt{e3nn} code.
\end{definition}

\begin{remark}
The spherical harmonics depends heavily on the choice of basis of the irreducible representation. $\tilde Y^l(\vec x) = A^l Y^l(\vec x)$ is the spherical harmonics of the irreps $A^l D^l(R) (A^l)^{-1}$ where $A^l$ is an invertible matrix.
\end{remark}

\begin{lemma}
Once the basis of irreps fixed, the spherical harmonics functions are unique for $\vec x$ on the unite sphere up to the choice of their sign $\pm Y^l$.
\end{lemma}

\begin{proof}
For all $\vec x \in \mathbb{R}^3$ there is a family of rotation that has no effect on $\vec x$ i.e. $R \vec x = \vec x$. Using the equivariance of the spherical harmonics it imply
\begin{equation}
    Y^l(\vec x) = D^l(R) Y^l(\vec x) \quad \forall R: R\vec x = x
\end{equation}
If we consider an infinitesimal rotation $D^l(R) = 1 + \epsilon J_{\vec x}^l$ we obtain the equation
\begin{equation} \label{eq:sh_sys}
    0 = J_{\vec x}^l Y^l(\vec x)
\end{equation}
where $J_{\vec x}^l$ is the infinitesimal rotation around $\vec x$.
The eigenvalues of the generators of rotations are $-il, \dots, -i, 0, i, \dots, il$, see proof in \cite{rao} page 157.
Therefore the space of solution of \ref{eq:sh_sys} is dimension 1 and $Y^l(\vec x)$ is unique up to a multiplicative constant.
Moreover any pair of vectors on the sphere can be related by a rotation and therefore $Y^l$ is unique up to a multiplicative constant.
Finally using that they have to be normalized $\|Y^l(\vec x)\|=1$ it fixes the multiplicative constant and we are only left with a choice of sign for each $l$.
\end{proof}

\begin{lemma}
When extended from the unit sphere to the full space $\mathbb{R}^3$, the spherical harmonics can be chosen to be polynomials of $\vec x$.
\end{lemma}

\begin{proof}
This proof require to be familiar with the tensor product introduced in Section~\ref{sec:tp}.
As you can see in Figure~\ref{fig:spherical_harmonics}, the successive tensor product of $\vec x$ with itself produce a row of polynomials that are yet to prove the spherical harmonics and rows of the same polynomials multiplied by $\|\vec x\|^{2n}$.
Since these polynomials satisfies the requirements to be spherical harmonics and the spherical harmonics are unique (up to a sign), we conclude that the spherical harmonics are polynomials on $\mathbb{R}^3$ constrained to be evaluated on the sphere.
As a sanity check, to see that all the other polynomials are zero due to symmetry (like for instance $\vec x \wedge \vec x = \vec 0$) we can make a counting argument. 
There is $(l+2)(l+1)/2$ polynomials of $\vec x$ of degree $l$ (choose 2 among $l+2$) and it match with the $(2l+1) + (2l-3) + (2l-7) + \dots$ polynomials of the form $\| \vec x \|^{2k} Y^{l-2k}(\vec x)$. 
\end{proof}

\begin{figure}[ht]
    \centering
    \def\svgwidth{11cm}
    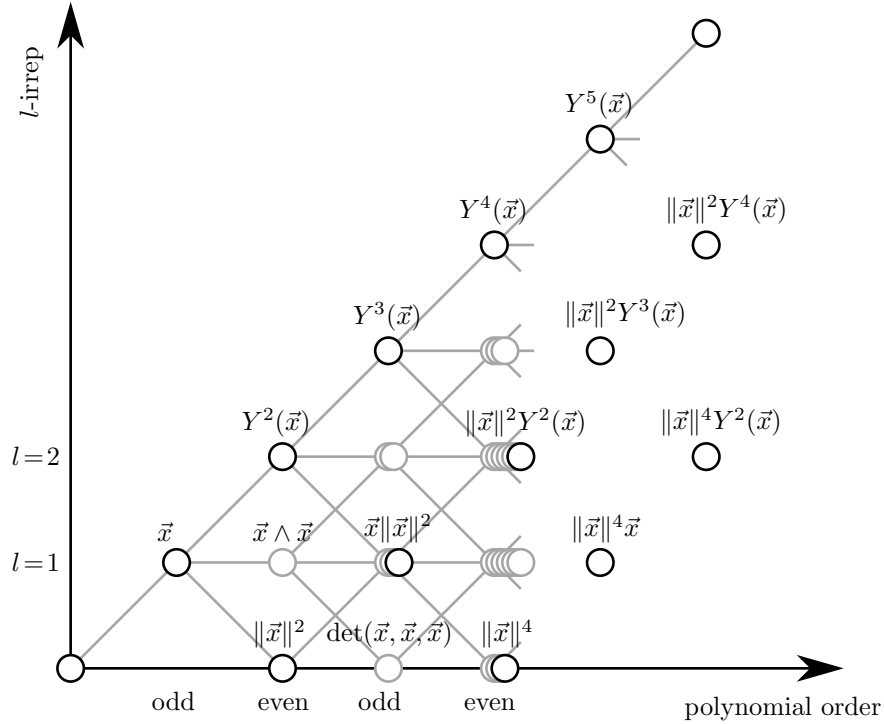
    \caption{Schematic representation of the polynomials obtained by applying successive tensor product with a vector $\vec x$.
    The spherical harmonics correspond to the highest irrep order of each polynomial order. 
    In gray one has the polynomials that would have been non-zero if the tensor product was taken with different arguments, for instance $\vec x\wedge \vec y$ is non-zero for arbitrary $\vec x$ and $\vec y$ but it is zero when considering $\vec x\wedge\vec x$.}
    \label{fig:spherical_harmonics}
\end{figure}

In \texttt{e3nn}, the choice of sign is made as follow $Y^{l+1}(\vec x) \propto C^{l+1, l, 1} Y^l(\vec x) Y^1(\vec x)$ where $C^{l+1, l, 1}$ are the Clebsh-Gordan coefficient introduced in section \ref{sec:clebsh}.
The spherical harmonics obtained this way matches the usual ones \cite{Rehfeld,sh_wikipedia_2021}.

\begin{remark}
The spherical harmonics are a basis for all the equivariant polynomials on the sphere (only on the sphere because we discarded the $\|\vec x\|^{2n} Y^l(\vec x)$ polynomials).
They also have the property that the frequency (oscillations of the function) increase with increasing $l$.
These two properties makes them a suited basis to represent smooth functions on the sphere.
\end{remark}

Spherical harmonics have many more interesting properties that we do not cover here.
One of them is that (with the appropriate normalization called \texttt{integral} in the code \texttt{e3nn}) they are orthogonal with respect to the following scalar product
\begin{equation}
    \langle f, g \rangle = \int_{s^2} f(\vec x) g(\vec x) d\vec x
\end{equation}

\subsection{Project to and from the sphere} \label{sec:s2}

Scalar functions on the sphere $f: s^2 \to \mathbb{R}$ can be expressed in the basis of the spherical harmonics
\begin{equation}
    f(\vec x) = \sum_{l=0}^\infty v^l \cdot Y^l(\vec x)
\end{equation}
where $v^l \in \mathbb{R}^{2l+1}$ are the coefficients of the function $f$ and $Y^l : \mathbb{R}^3 \to \mathbb{R}^{2l+1}$ are the spherical harmonics.
It's easy to see that under the action of rotation, the coefficients transforms as irreps.
\begin{proof}
\begin{align}
    [R f](\vec x) &= f(R^{-1} \vec x) && \text{(action of rotation on $f$)} \\
            &= \sum_{l=0}^\infty v^l \cdot Y^l(R^{-1} \vec x) \\
            &= \sum_{l=0}^\infty v^l \cdot (D^l(R^{-1}) Y^l(\vec x)) \\
            &= \sum_{l=0}^\infty (D^l(R) v^l) \cdot Y^l(\vec x)
\end{align}
\end{proof}
Truncating the coefficients to a maximum $L$ cuts off the high frequencies of the signal.
\texttt{e3nn} contains efficient functions to convert back and forth between $\{v^l\}_{l=0}^L$ and a discretized version of $f$ evaluated on a grid.
The grid we use is chosen to improve the performance of the transformation\cite{DRISCOLL1994202}.
The grid is made of necklaces around the $y$ axis, see Figure~\ref{fig:s2grid}.

\begin{figure}[ht]
    \centering
    \includegraphics[width=0.6\linewidth]{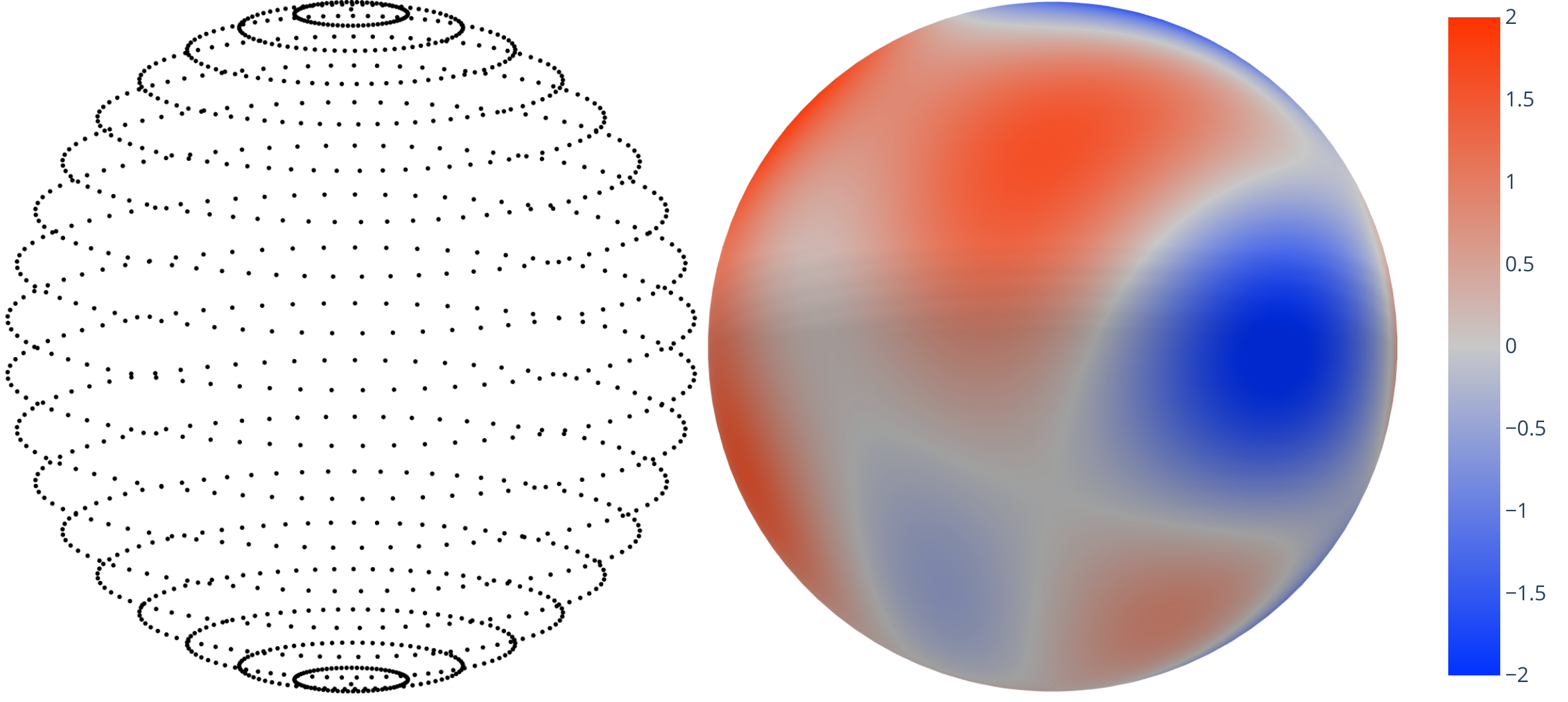}
    \caption{\textbf{Left} Grid on the sphere used in \texttt{e3nn}. \textbf{Right} Random signal on the sphere (cutoff at $L=5$).}
    \label{fig:s2grid}
\end{figure}



\section{Tensor Product} \label{sec:tp}
The tensor product is the equivariant multiplication operation for two representations. 
It is denoted as
\begin{equation}
    x \otimes y.
\end{equation}
Depending on the ``type" of $x$ and $y$ this operations takes different forms. In all cases, it satisfies the two conditions: (i) it is bilinear (ii) and equivariant. For instance if $x$ and $y$ are two scalars it will simply be the product $xy$ (up to a multiplicative constant). If $x$ is a scalar and $y$ is a vector, $x\otimes y$ will be the vector made of the product of $x$ with $y$. Note that the ``type" of $x\otimes y$ depends on the type of $x$ and $y$. In the last example, if $x$ is a pseudoscalar and $y$ is a vector, $x\otimes y$ is a pseudovector.
In the case of two vectors, $x\otimes y$ is an element of $\mathbb{R}^9$ ($x\otimes y \in \mathbb{R}^{\dim(x)\dim(y)}$). Via a change of basis $x\otimes y$ can be decomposed into a scalar (corresponding to the scalar product), a pseudovector (the cross product) and 5 component transforming according to the even irrep of $l=2$.
These branching operations are independently bilinear and equivariant.
When computing the tensor product of representations of higher $l$ we obtain more independent operations. We call these independent operations \textbf{paths}. The formulas
\begin{equation} \label{eq:l1l2l3}
    \left\{
    \begin{array}{l}
        |l_1 - l_2| \leq l_3 \leq l_1 + l_2 \\
        p_1 p_2 = p_3
    \end{array}
    \right.
\end{equation}
gives all the paths we get when computing the tensor product of an irrep $(l_1,p_1)$ with an irrep $(l_2,p_2)$, there is one path for each allowed value of $l_3$ ($l_1 + l_2 - |l_1 - l_2| + 1$ in total).
In \texttt{e3nn} we relax the definition of tensor product as any bilinear, equivariant operation that performs on irreps and outputs irreps. 
The paths (constrained by \eqref{eq:l1l2l3}) are the building blocks of our tensor product operations.
We allow for linear combinations of those paths mixing different inputs together.
Let's consider an example. If one has two vectors $(\vec x_1, \vec x_2)$ multiplied by a scalar and a vector $(y_1, \vec y_2)$ (here we add arrows in the notation in order to remember the type but we usually don't do it) we can for instance consider the following operation:
\begin{equation} \label{eq:tp_example}
    (x, y) \mapsto z = \left(
    \begin{array}{ccc}
        w_1 \vec x_1 \cdot \vec y_2 &+& w_2 \vec x_2 \cdot \vec y_2 \\
        w_3 \vec x_1 y_1 &+& w_4 \vec x_2 y_1
    \end{array}
    \right)
\end{equation}
One can check that this operation is bilinear and equivariant. It is also parametrized by the weights $w_1, w_2, w_3, w_4$.
\begin{figure}[ht]
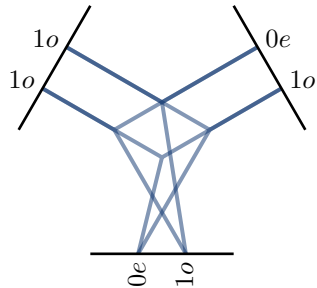

    \centering
\begingroup%
\makeatletter%
\begin{pgfpicture}%
\pgfpathrectangle{\pgfpointorigin}{\pgfqpoint{2.000000in}{2.000000in}}%
\pgfusepath{use as bounding box, clip}%
\begin{pgfscope}%
\pgfsetbuttcap%
\pgfsetmiterjoin%
\pgfsetlinewidth{0.000000pt}%
\definecolor{currentstroke}{rgb}{1.000000,1.000000,1.000000}%
\pgfsetstrokecolor{currentstroke}%
\pgfsetstrokeopacity{0.000000}%
\pgfsetdash{}{0pt}%
\pgfpathmoveto{\pgfqpoint{0.000000in}{0.000000in}}%
\pgfpathlineto{\pgfqpoint{2.000000in}{0.000000in}}%
\pgfpathlineto{\pgfqpoint{2.000000in}{2.000000in}}%
\pgfpathlineto{\pgfqpoint{0.000000in}{2.000000in}}%
\pgfpathclose%
\pgfusepath{}%
\end{pgfscope}%
\begin{pgfscope}%
\pgfpathrectangle{\pgfqpoint{0.198490in}{0.242174in}}{\pgfqpoint{1.603020in}{1.607826in}}%
\pgfusepath{clip}%
\pgfsetbuttcap%
\pgfsetmiterjoin%
\pgfsetlinewidth{1.505625pt}%
\definecolor{currentstroke}{rgb}{0.031373,0.188235,0.419608}%
\pgfsetstrokecolor{currentstroke}%
\pgfsetstrokeopacity{0.500000}%
\pgfsetdash{}{0pt}%
\pgfpathmoveto{\pgfqpoint{1.125032in}{0.396404in}}%
\pgfpathlineto{\pgfqpoint{0.749937in}{1.046087in}}%
\pgfpathmoveto{\pgfqpoint{0.374842in}{1.262648in}}%
\pgfpathlineto{\pgfqpoint{0.749937in}{1.046087in}}%
\pgfpathmoveto{\pgfqpoint{1.500126in}{1.479209in}}%
\pgfpathlineto{\pgfqpoint{0.749937in}{1.046087in}}%
\pgfusepath{stroke}%
\end{pgfscope}%
\begin{pgfscope}%
\pgfpathrectangle{\pgfqpoint{0.198490in}{0.242174in}}{\pgfqpoint{1.603020in}{1.607826in}}%
\pgfusepath{clip}%
\pgfsetbuttcap%
\pgfsetmiterjoin%
\pgfsetlinewidth{1.505625pt}%
\definecolor{currentstroke}{rgb}{0.031373,0.188235,0.419608}%
\pgfsetstrokecolor{currentstroke}%
\pgfsetstrokeopacity{0.500000}%
\pgfsetdash{}{0pt}%
\pgfpathmoveto{\pgfqpoint{0.874968in}{0.396404in}}%
\pgfpathlineto{\pgfqpoint{1.000000in}{0.901713in}}%
\pgfpathmoveto{\pgfqpoint{0.374842in}{1.262648in}}%
\pgfpathlineto{\pgfqpoint{1.000000in}{0.901713in}}%
\pgfpathmoveto{\pgfqpoint{1.625158in}{1.262648in}}%
\pgfpathlineto{\pgfqpoint{1.000000in}{0.901713in}}%
\pgfusepath{stroke}%
\end{pgfscope}%
\begin{pgfscope}%
\pgfpathrectangle{\pgfqpoint{0.198490in}{0.242174in}}{\pgfqpoint{1.603020in}{1.607826in}}%
\pgfusepath{clip}%
\pgfsetbuttcap%
\pgfsetmiterjoin%
\pgfsetlinewidth{1.505625pt}%
\definecolor{currentstroke}{rgb}{0.031373,0.188235,0.419608}%
\pgfsetstrokecolor{currentstroke}%
\pgfsetstrokeopacity{0.500000}%
\pgfsetdash{}{0pt}%
\pgfpathmoveto{\pgfqpoint{1.125032in}{0.396404in}}%
\pgfpathlineto{\pgfqpoint{1.000000in}{1.190461in}}%
\pgfpathmoveto{\pgfqpoint{0.499874in}{1.479209in}}%
\pgfpathlineto{\pgfqpoint{1.000000in}{1.190461in}}%
\pgfpathmoveto{\pgfqpoint{1.500126in}{1.479209in}}%
\pgfpathlineto{\pgfqpoint{1.000000in}{1.190461in}}%
\pgfusepath{stroke}%
\end{pgfscope}%
\begin{pgfscope}%
\pgfpathrectangle{\pgfqpoint{0.198490in}{0.242174in}}{\pgfqpoint{1.603020in}{1.607826in}}%
\pgfusepath{clip}%
\pgfsetbuttcap%
\pgfsetmiterjoin%
\pgfsetlinewidth{1.505625pt}%
\definecolor{currentstroke}{rgb}{0.031373,0.188235,0.419608}%
\pgfsetstrokecolor{currentstroke}%
\pgfsetstrokeopacity{0.500000}%
\pgfsetdash{}{0pt}%
\pgfpathmoveto{\pgfqpoint{0.874968in}{0.396404in}}%
\pgfpathlineto{\pgfqpoint{1.250063in}{1.046087in}}%
\pgfpathmoveto{\pgfqpoint{0.499874in}{1.479209in}}%
\pgfpathlineto{\pgfqpoint{1.250063in}{1.046087in}}%
\pgfpathmoveto{\pgfqpoint{1.625158in}{1.262648in}}%
\pgfpathlineto{\pgfqpoint{1.250063in}{1.046087in}}%
\pgfusepath{stroke}%
\end{pgfscope}%
\begin{pgfscope}%
\pgfpathrectangle{\pgfqpoint{0.198490in}{0.242174in}}{\pgfqpoint{1.603020in}{1.607826in}}%
\pgfusepath{clip}%
\pgfsetbuttcap%
\pgfsetmiterjoin%
\pgfsetlinewidth{1.003750pt}%
\definecolor{currentstroke}{rgb}{0.000000,0.000000,0.000000}%
\pgfsetstrokecolor{currentstroke}%
\pgfsetdash{}{0pt}%
\pgfpathmoveto{\pgfqpoint{1.750190in}{1.046087in}}%
\pgfpathlineto{\pgfqpoint{1.375095in}{1.695770in}}%
\pgfpathmoveto{\pgfqpoint{0.624905in}{1.695770in}}%
\pgfpathlineto{\pgfqpoint{0.249810in}{1.046087in}}%
\pgfpathmoveto{\pgfqpoint{0.624905in}{0.396404in}}%
\pgfpathlineto{\pgfqpoint{1.375095in}{0.396404in}}%
\pgfusepath{stroke}%
\end{pgfscope}%
\begin{pgfscope}%
\definecolor{textcolor}{rgb}{0.000000,0.000000,0.000000}%
\pgfsetstrokecolor{textcolor}%
\pgfsetfillcolor{textcolor}%
\pgftext[x=0.333175in,y=1.262648in,right,base]{\color{textcolor}\sffamily\fontsize{10.000000}{12.000000}\selectfont \(\displaystyle 1o\)}%
\end{pgfscope}%
\begin{pgfscope}%
\definecolor{textcolor}{rgb}{0.000000,0.000000,0.000000}%
\pgfsetstrokecolor{textcolor}%
\pgfsetfillcolor{textcolor}%
\pgftext[x=0.458207in,y=1.479209in,right,base]{\color{textcolor}\sffamily\fontsize{10.000000}{12.000000}\selectfont \(\displaystyle 1o\)}%
\end{pgfscope}%
\begin{pgfscope}%
\definecolor{textcolor}{rgb}{0.000000,0.000000,0.000000}%
\pgfsetstrokecolor{textcolor}%
\pgfsetfillcolor{textcolor}%
\pgftext[x=1.541793in,y=1.479209in,left,base]{\color{textcolor}\sffamily\fontsize{10.000000}{12.000000}\selectfont \(\displaystyle 0e\)}%
\end{pgfscope}%
\begin{pgfscope}%
\definecolor{textcolor}{rgb}{0.000000,0.000000,0.000000}%
\pgfsetstrokecolor{textcolor}%
\pgfsetfillcolor{textcolor}%
\pgftext[x=1.666825in,y=1.262648in,left,base]{\color{textcolor}\sffamily\fontsize{10.000000}{12.000000}\selectfont \(\displaystyle 1o\)}%
\end{pgfscope}%
\begin{pgfscope}%
\definecolor{textcolor}{rgb}{0.000000,0.000000,0.000000}%
\pgfsetstrokecolor{textcolor}%
\pgfsetfillcolor{textcolor}%
\pgftext[x=0.913285in, y=0.220622in, left, base,rotate=90.000000]{\color{textcolor}\sffamily\fontsize{10.000000}{12.000000}\selectfont \(\displaystyle 0e\)}%
\end{pgfscope}%
\begin{pgfscope}%
\definecolor{textcolor}{rgb}{0.000000,0.000000,0.000000}%
\pgfsetstrokecolor{textcolor}%
\pgfsetfillcolor{textcolor}%
\pgftext[x=1.163348in, y=0.217970in, left, base,rotate=90.000000]{\color{textcolor}\sffamily\fontsize{10.000000}{12.000000}\selectfont \(\displaystyle 1o\)}%
\end{pgfscope}%
\end{pgfpicture}%
\makeatother%
\endgroup%
    \caption{Top left represent the first inputs ($x$ in \eqref{eq:tp_example}). Top right represent the second input ($y$) and the bottom part represent the output ($z$). The blue connections represent the paths: there is two paths going into each output.}
    \label{fig:tp_example}
\end{figure}

\noindent We will call this operation a \textit{tensor product} although it is not the $\otimes$ operation.
This operations contains 4 paths: (i) from $x_1$ times $y_2$ to the first output $z_1$ of irreps $l=1$ times $l=1$ into $l=0$ (a scalar product), (ii) $x_2 \times y_2 \to z_1$ of type $1 \times 1 \to 0$, (iii) $x_1 \times y_1 \to z_2$ of type $1 \times 0 \to 1$ and (iv) $x_2 \times y_1 \to z_2$ of type $1 \times 0 \to 1$.
These paths are represented schematically in Figure~\ref{fig:tp_example}.
In \texttt{e3nn} this operation can be constructed as follow:
\begin{lstlisting}[language=Python]
import torch
from e3nn import o3

irreps_in1 = o3.Irreps('1o + 1o')
irreps_in2 = o3.Irreps('0e + 1o')
irreps_out = o3.Irreps('0e + 1o')

tp = o3.FullyConnectedTensorProduct(irreps_in1, irreps_in2, irreps_out)
# FullyConnectedTensorProduct is a subclass of TensorProduct

x1 = torch.tensor([1., 0., 0.,    2., 1., 0.])
x2 = torch.tensor([123.4,         0., 1., 5.])
out = tp(x1, x2)  # evaluate the tensor product
\end{lstlisting}

\texttt{e3nn.o3.FullyConnectedTensorProduct} is a simplified interface to a more general operation called \texttt{e3nn.o3.TensorProduct} who is defined by
\begin{enumerate}
    \item A set of irreps for the first input
    \item A set of irreps for the second input
    \item A set of irreps for the output
    \item A set of paths, each path is defined as
    \begin{enumerate}
        \item Which two inputs are connected to which output (it has to satisfy \eqref{eq:l1l2l3})
        \item It's weight value (it can be a learned parameter)
    \end{enumerate}
\end{enumerate}
We typically initialize the weights as some normalization constant multiplying a learned parameter initialized randomly with a normalized Gaussian. The normalization constant is carefully chosen such that every path contributes equally and that the output amplitude is close to one. See Section~\ref{app:tp} and Section~\ref{sec:equiv_modules}.

\subsection{Clebsch–Gordan coefficients} \label{sec:clebsh}
The Clebsch–Gordan coefficients are the change of basis that decompose the tensor product $\otimes$ into irreps.
The Clebsch–Gordan coefficients (also called Wigner 3j symbols) satisfy the following equation
\begin{equation}
    C^{l_1, l_2, l_3}_{lmn} = C^{l_1, l_2, l_3}_{ijk} D^{l_1}_{il}(R) D^{l_2}_{jm}(R) D^{l_3}_{kn}(R) \qquad \forall R \in SO(3)
\end{equation}

In the literature, they are usually computed for the complex basis that diagonalize one of the generator (see discussion in Section~\ref{sec:repr}). We get the Clebsch–Gordan coefficients from QuTip\cite{qutip} and we apply the change of basis \ref{eq:real_complex} to convert them in our real basis.

\subsection{Different Paths in the Tensor Product} \label{app:tp}
Each output of an \texttt{e3nn} tensor product is a sum of \textit{paths}.
\begin{equation}
    \text{out} = \sqrt{\frac{1}{\text{number of paths}}} \sum_{i\in\text{paths}} \text{value of path $i$}
\end{equation}
where the multiplicative constant ensure that the components of the output have variance 1.
Each path can then be written as
\begin{align} \label{eq:tp_path}
    (\text{value of path})_{wk} & = \nonumber\\ 
    & \frac{1}{\sqrt{m_1 m_2}} \sum_{u=1}^{m_1} \sum_{v=1}^{m_2} \sum_{i=1}^{2 l_1 + 1} \sum_{j=1}^{2 l_2 + 1} \nonumber\\
    & w_{uvw} C_{ijk} (\text{in1})_{ui} (\text{in2})_{vj}
\end{align}
where $m_1$ (resp. $m_2$) is the multiplicity of the first (resp. second) input.
$l_1$ (resp. $l_2$) is the representation order of the first (resp. second) input.
$w_{uvw}$ are the weights, they can be learned parameters. $C_{ijk}$ are the Clebsch-Gordan coefficients i.e. the orthogonal change of variable between the natural tensor product $\otimes$ basis and the irrep of the output.

We implement also more sparse version of the weight connections. Equation \eqref{eq:tp_path} corresponds to what we call \texttt{uvw} weight connection mode. While for instance another type of connection mode that is provided is the \texttt{uvu} connection mode implemented as follow
\begin{align}
    (\text{value of path})_{uk} & = \nonumber\\ 
    & \frac{1}{\sqrt{m_2}} \sum_{v=1}^{m_2} \sum_{i=1}^{2 l_1 + 1} \sum_{j=1}^{2 l_2 + 1} \nonumber\\
    & w_{uv} C_{ijk} (\text{in1})_{ui} (\text{in2})_{vj}.
\end{align}
Note that $\sqrt{m_1 m_2}$ has been replaced by $\sqrt{m_2}$ because we don't sum over $u$ anymore. Other modes like \texttt{uuu} and \texttt{uvuv} are described in the library.

\subsection{Generality of Tensor Product} \label{sec:tp_general}

\begin{lemma}
The \texttt{TensorProduct} class of \texttt{e3nn} can represent any bi-linear equivariant operations combining two set of irreps into irreps. 
\end{lemma}
\begin{proof}
The tensor product $x \otimes y$ is the complete bi-linear operation. In the sense that $\dim(x \otimes y) = \dim(x) \dim(y)$ and it is non trivial (not equal to zero). Then once the basis of irreps chosen, the decomposition into irreps is unique. In \texttt{TensorProduct} all the paths can be created and weight therefore all possible operations can be performed.
\end{proof}

\subsection{Reducing Tensor Products into irreps} \label{sec:rtp}
A common question that arises when dealing with physical systems is how to re-express the interaction of multiple objects into its most basic components. For example, how to deduce that a rank 4 symmetric Cartesian tensor only has 21 (rather than $3^4=81$) degrees of freedom. Our \texttt{ReducedTensorProduct} operations is able to readily perform these calculations. 

Our \texttt{ReducedTensorProduct} class takes as arguments an index formula which can specify the symmetries under permutation of the indices and the irrep decomposition of each index. Considering the same example, a symmetry rank 4 tensor would be articulated as
\begin{lstlisting}[language=Python]
from e3nn.o3 import ReducedTensorProducts

r = ReducedTensorProducts('ijkl=jikl=ijlk=klij', i='1o')
print(r.irreps_out)
# 2x0e+2x2e+1x4e
\end{lstlisting}

The two symmetries, (i) permutation of indices and (ii) representation of $SO(3)$ commute. Therefore the two problems can be solved independently and merged at the end.

Let $X$ being the vector space of the tensor we want to reduce. The dimension of $X$ is the product of the dimensions of the representations of the indices of the tensor.

Let $\mathcal P$ be the group of indices permutation obtained from the index formula (for example \texttt{ijkl=jikl=klij} germinate a group of 8 permutations).

Both $SO(3)$ and $P$ act on $X$ with the representation $D_X(g), g\in SO(3)$ and $D_X(\tau), \tau \in \mathcal P$ respectively.

And let $D_S: \mathcal P \to \{\pm1\}$ be the representation encoding the formula signs, $D_S(\tau)$ is the sign acting on the tensor when permuting it's indices by $\tau$.

The first step of the algorithm is to find a basis $P\in \mathbb{R}^{p\times \dim(X)}$ of the sub vector space stable to permutations:
\begin{equation}
    D_S(\tau) P = P D_X(\tau) \qquad \tau \in \mathcal P 
\end{equation}

\begin{lstlisting}[language=Python]
# For instance with ij=ji and tensor of shape 2x2
P = [
    [1, 0,  0, 0],
    [0, 0,  0, 1],
    [0, 1,  1, 0],
]
\end{lstlisting}
Then let's only consider the rotations\footnote{or any other group as long as we know how to decompose the tensor product into irreps.} and leave aside the indices permutation constraints. Performing the tensor products and decomposing into irreps in chain leads to one basis $R_i$ per irrep $i$ (one base for the odd scalars, one for the even scalars, another one for the odd vectors and so on).

\begin{equation}
    [1 \otimes D_i(g)] R_i = R_i D_X(g) \qquad g\in SO(3)
\end{equation}
where $D_i$ is the Wigner D matrix of irrep $i$ and $1$ is the identity of dimension equal to the multiplicity of $i$ in $X$.

The sum of the dimensions of the $R_i$ bases is equal to the dimension of the space $X$. While the dimension of the $P$ basis is smaller than the tensor dimension (due to the permutation constraints).

The last step of the procedure consist of, for each irrep $i$, finding linear combinations of the vectors of the $R_i$ basis such that they lie in the span of the basis $P$. The coefficient matrix $W_i$ mixes the multiplicities of $R_i$ to create the bases $Q_i$,
\begin{equation}
    Q_i = [W_i \otimes 1] R_i
\end{equation}
such that they satisfy
\begin{equation}
    \left\{ \begin{array}{ll}
        [1 \otimes D_i(g)] Q_i = Q_i D_X(g) & g\in SO(3) \\
        D_S(\tau) Q_i = Q_i D_X(\tau) & \tau \in \mathcal{P}
    \end{array}
    \right.
\end{equation}

The first row is trivially satisfied because $1\otimes D_i(g)$ commutes with $W_i\otimes 1$.
The $W_i$ satisfying the second row can be found using $P$,
\begin{equation}
    Q_i = [W_i \otimes 1] R_i = \tilde W_i P
\end{equation}

This is done by solving a linear system described in the pseudo code below.
\begin{lstlisting}[language=Python]
ir = "1o"
r = R[ir][:, 0, :]  # multiplicity x dim(X)
P  # p x dim(X)

U = matrix_by_blocks(
   r @ r.T,    -r @ P.T,
  -P @ r.T,     P @ P.T
)
E, V = torch.linalg.eigh(U)
W = V[:, E < 1e-9][:multiplicity].T  # num_solutions x multiplicity
W = Gram_Schmidt(W.T @ W)  # num_solutions x multiplicity

Q[ir] = torch.einsum('ij,jkx->ikx', W, R[ir])
\end{lstlisting}

Note that in the pseudo code the system is solved only for one component of the irrep (the first one in this pseudo code).
In principle one should compute the intersections of all the solutions.
It is not necessary because $Q$ has to be a basis for the same subspace as $P$, therefore all degrees of freedom of $P$ have to be mapped to $Q$.
We claim that $W$ computed with a single component cannot contain more solutions than the intersection of all solutions computed with each component, this is true because otherwise it would imply that some directions in the space span by $P$ map to a fraction of an irrep. This is not possible because the tensor we reduce is a representation and the irrep are irreducible.

The final change of basis $Q$ is made by concatenating the all the $Q_i$ we obtained.








  
    


For the special case of a tensor made out of the product of vectors,
we alternatively keep track of the tensor product operations in order to implement the change of basis by performing those tensor products instead of applying the change of basis with a matrix multiplication.




\section{Equivariant modules} \label{sec:equiv_modules}

Using \texttt{TensorProduct} and spherical harmonics we implemented the point convolution of \cite{tfn}, the voxel convolution of \cites{3dsteerable} and the equivariant transformer of \cite{se3_transformers}.

Here is an example of a module computing an equivariant polynomial $P$ of a set of positions $\{\vec x_i\}_i$. This polynomial is parameterized by weights. For any value of the weights it is equivariant with respect to all transformation of $E(3)$:
\begin{equation}
    \left\{ \begin{array}{l}
        P(\{R \vec x\}) = D(R) P(\{\vec x\}) \\
        P(\{\vec x + \vec a\}) = P(\{\vec x\}) \\
        P(\{-\vec x\}) = D(-1) P(\{\vec x\})
    \end{array}
    \right.
\end{equation}

\begin{lstlisting}[language=Python]
import torch
import e3nn

def scatter(src, index, dim_size):
    out = src.new_zeros(dim_size, src.shape[1])
    index = index.reshape(-1, 1).expand_as(src)
    return out.scatter_add_(0, index, src)

def radius_graph(pos, r_max):
    r = torch.cdist(pos, pos)
    index = ((r < r_max) & (r > 0)).nonzero().T
    return index

class Polynomial(torch.nn.Module):
    def __init__(self, irreps_out):
        super().__init__()
        self.irreps_sh = e3nn.o3.Irreps.spherical_harmonics(3)
        irreps_mid = e3nn.o3.Irreps(
            "64x0e + 24x1e + 24x1o + 16x2e + 16x2o")

        self.tp1 = e3nn.o3.FullyConnectedTensorProduct(
            irreps_in1=self.irreps_sh,
            irreps_in2=self.irreps_sh,
            irreps_out=irreps_mid,
        )
        self.tp2 = e3nn.o3.FullyConnectedTensorProduct(
            irreps_in1=irreps_mid,
            irreps_in2=irreps_mid,
            irreps_out=irreps_out,
        )
        self.irreps_out = self.tp2.irreps_out

    def forward(self, pos, max_radius, num_neigh, num_nodes):
        n = pos.shape[0]  # number of nodes

        # Compute tensors of indices representing the graph
        e_src, e_dst = radius_graph(pos, max_radius)

        e_x = e3nn.o3.spherical_harmonics(
            l=self.irreps_sh,
            x=pos[e_src] - pos[e_dst],
            normalize=False,  # we want polynomials of x
            normalization='component'
        )
        
        n_x = scatter(e_x, e_dst, n) / num_neigh**0.5
        e_x = self.tp1(n_x[e_src], e_x)

        n_x = scatter(e_x, e_dst, n) / num_neigh**0.5
        e_x = self.tp2(n_x[e_src], e_x)

        return e_x.sum(0) / num_neigh**0.5 / num_nodes**0.5
\end{lstlisting}

\subsection{Initialization} \label{sec:init}

In all the modules, we initialize the weights with normalized Gaussian of mean 0 and variance 1. Also we chose to statistically favor a particular normalization of our data. 
We call it \texttt{component} normalization: if $x \in \mathbb{R}^d$ we aim to have $\| x \|^2 = d$, such that in average all component have magnitude $\sim 1$ (we say "all the component are of order 1", note that here the term order has nothing to do with the irreps).

Our initialization scheme satisfies that at initialization (i) all the preactivation have mean 0 and variance 1 (ii) all the post activations have the second moment 1 (iii) all the layers learn when the width (i.e. the multiplicities) is sent to infinity.

(iii) is satisfied because our initialization satisfies the Maximal Update Parametrization \cite{yang2020feature}. We use $\mu$P shifted by $\theta=1/2$ according to Symmetries of abc-Parametrizations in \cite{yang2020feature}. All our parameters are initialized with normal distribution $\mathcal{N}(0, 1)$. The learning rate is set proportional to the number of features. And multiplicative constants are put to satisfy both (i) and abc-Parametrizations. 

For example, for a fully connected layer
\begin{equation}
    z = \phi \left(\frac{1}{\sqrt{d}} W x \right)
\end{equation}
where $z \in \mathbb{R}^n$ is a post activation, $W \in \mathbb{R}^{n\times d}$ are the parameters and $x \in \mathbb{R}^d$ are the inputs or post activations of the previous layer. 
Following our initialization scheme each component of $W$ is initialized with $\mathcal{N}(0, 1)$.
Each component of the matrix-vector product $W x$ is of variance $n$. The multiplicative constant $\frac{1}{\sqrt{n}}$ bring the variance to 1.
$\phi$ is an activation function that if necessary has been rescaled such that $\frac{1}{\sqrt{2 \pi}}\int \phi(x)^2 e^{-x^2/2} dx = 1$.
Thanks to this rescaling, each component of $z$ has a second moment of 1.

\section{Data Efficiency}

\begin{figure}[ht]
    \centering
    \includegraphics[width=6cm]{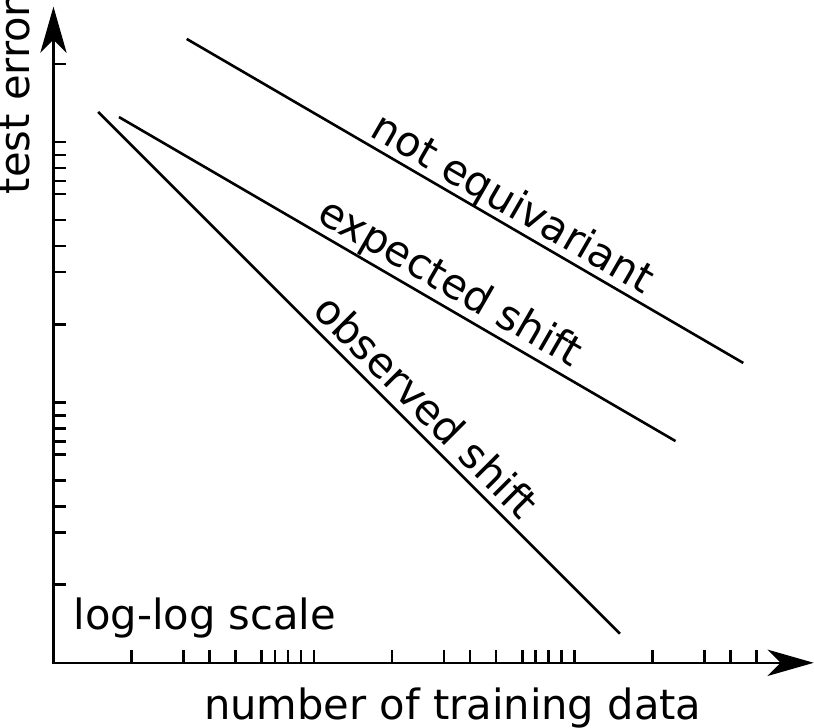}
    \caption{The data-efficiency of equivariant neural network is more than a multiplicative constant, it affects the power law exponent of the learning curve.}
    \label{fig:data_efficiency}
\end{figure}

One could argue that invariance and equivariance allow for better data-efficiency because it remove the need for data-augmentation and then the number of training sample would be divided by the number of augmentation needed to observe all the orientations.
If it was just that, we would observe a shift in the learning curve, because for the same performance we would need $c$ times less training data.

Surprisingly it has been observed\cite{md_batzner2021} that the exponent of the leaning curve (test error as a function of the number of training data) is different for equivariant neural networks.
This new behavior can be visualized as a change of slope in a log-log plot, see Figure~\ref{fig:data_efficiency}.

Unfortunately we have no theoretical explanation for this change of exponent.

\section{Expressing previous works with e3nn}

The primary difference between how Tensor Field Networks and 3D Steerable CNNs are articulated is that the filters in 3D Steerable CNNs are articulated as the right operation of the tensor product on the spherical harmonic expansion where as in tensor field networks the filter is spherical harmonic expansion. This is due to the fact that 3D Steerable CNNs was developed for voxels and by contracting one of the input indices, the form of the filter is compatible with a standard CNN module.

However, in order to reduce computational overhead when operating on point clouds, Tensor Field Networks does not add weights in the tensor product and performs a linear operation after the tensor product — analogous to SchNet and the idea of matrix decomposition. Implemented in e3nn as connection modes ‘uvu’ vs. ‘uvw’.

Spherical CNNs \cite{s.2018spherical} operate on signals on a sphere they can be articulated in real space or alternatively in irrep space. In irrep space the convolution maps to \texttt{Linear} layers and the activations function require a Fourier transform and reverse transform.

Clebsch-Gordan Networks \cite{kondor2018clebsch} use \texttt{TensorSquare} as a nonlinearity followed by a \texttt{Linear} layer.


Ref. \cite{se3_transformers} implements an $SE(3)$ equivariant transformer.

All these models can be constructed using e3nn primitives.



\section{Conclusion}
\texttt{e3nn} is a general framework for articulating composable $E(3)$ equivariant operations for learning on 3D data and beyond. 

By focusing on building expressive general operations, we can efficiently express complex equivariant functions and build sophisticated trainable modules. Our framework does not impose specific training operations but rather gives the user all needed building blocks to compose their own modules, much as how PyTorch or Tensorflow provide a flexible autodifferentiation API of standard matrix operations. Because we have distilled all equivariant operations to our core building blocks, all optimization efforts can be focused on these few classes.

Models built with \texttt{e3nn} using high order representations ($l>1$) are more data-efficient than models limited to scalars and vectors \cite{md_batzner2021}.

\printbibliography

\appendix

\end{document}